\documentclass[letterpaper, 10 pt, conference]{ieeeconf}  
\usepackage[tbtags]{amsmath}
\usepackage{amssymb}
\usepackage{booktabs}
\usepackage{xurl}
\usepackage[hidelinks,breaklinks=true]{hyperref}
\usepackage[caption=false,font=footnotesize]{subfig}
\usepackage{graphicx}
\usepackage{algorithm,algorithmic}
\newtheorem{theorem}{Theorem}
\newtheorem{lemma}{Lemma}
\newtheorem{assumption}{Assumption}

\usepackage[noadjust]{cite}

\IEEEoverridecommandlockouts                              

\usepackage{xcolor}
\newcommand{\name}{\textsf{HardNet++}}

\newcommand{\revised}[1]{{\color{black}#1}}
\newcommand{\clip}{\operatorname{clip}}

\title{\LARGE \bf
\revised{HardNet++:} Nonlinear Constraint Enforcement in Neural Networks
}
\author{Andrea Goertzen, Kaveh Alim, Youngjae Min, and Navid Azizan
\vspace{0.5em}\\Massachusetts Institute of Technology
}

\begin{document}

\maketitle
\thispagestyle{empty}
\pagestyle{empty}

\begin{abstract}

Enforcing constraint satisfaction in neural network outputs is critical for safety, reliability, and physical fidelity in many control and decision-making applications. While soft-constrained methods penalize constraint violations during training, they do not guarantee constraint adherence during inference. Other approaches guarantee constraint satisfaction via a projection layer, \revised{but often rely on the existence of a tractable projection onto the feasible set, limiting their utility in more general problem settings.} Many real-world problems of interest are nonlinear \revised{and lack the special structure admitting a tractable projection}, motivating the development of methods that can enforce \revised{general} nonlinear constraints. 
To this end, we introduce \revised{\name}, a constraint-satisfaction method that enforces linear and nonlinear equality and inequality constraints. Our approach iteratively adjusts the network output via damped local linearizations of the constraints. \revised{Each iteration is differentiable, admitting an end-to-end training framework, where the constraint satisfaction layer is active during training.} We show that under certain regularity conditions, this procedure enforces nonlinear constraint satisfaction to arbitrary tolerance. Finally, we \revised{demonstrate tight constraint adherence without loss of optimality in a learning-for-optimization context, where we apply this method to a nonlinear model predictive control problem.\footnote{The code is available at: \url{https://github.com/andreagoertzen/HardNet-plus-plus}}} 

\end{abstract}

\section{INTRODUCTION}\label{sec:intro}
Neural networks have demonstrated remarkable success in modeling complex physical dynamics and learning surrogate representations of computational processes, often with improved efficiency compared to traditional numerical methods. As a result, they are increasingly used in scientific machine learning, where models are expected to represent underlying physical systems. In these settings, outputs must often satisfy domain-specific constraints derived from first principles, such as conservation of mass or energy, symmetry relationships, or governing equations. Similarly, safety-critical applications require strict enforcement of bounds on process variables or stability constraints. Standard neural networks, however, do not inherently guarantee adherence to such constraints, which can lead to physically inconsistent or infeasible predictions, limiting the practicality of such models. This limitation motivates the development of principled constraint enforcement methods that preserve feasibility in learned models. 

Beyond scientific modeling, neural networks are also increasingly used to approximate optimization solvers. Many applications, including model predictive control, resource allocation, and real-time decision making, require repeatedly solving high-dimensional optimization problems online. Although modern optimization solvers are highly optimized, repeated real-time solutions can still be computationally expensive, motivating the use of neural networks as fast surrogate solvers. However, these optimization problems typically include equality and inequality constraints that define the feasible solution set. Consequently, the practical deployment of neural network surrogate solvers depends on the ability to ensure that predicted solutions remain feasible.

Despite the importance of constraint satisfaction in both scientific machine learning and learned optimization, enforcing constraints in neural network outputs remains challenging. Soft-constrained approaches encourage constraint satisfaction by penalizing violations during training \cite{raissi2019pinn}, but they generally cannot guarantee feasibility at inference time, particularly for inputs outside the training distribution. To address this limitation, hard-constrained methods have been developed to enforce feasibility by design \cite{min2024hardnet}. This is often done via specific parameterizations \cite{balestriero2023police} or projection-based methods that project the network output onto a feasible set defined by the constraints. For linear \emph{equality} constraints, a closed-form projection exists to minimize the Euclidean distance between the network output and a projected feasible point \cite{chen2024physics}. HardNet \cite{min2024hardnet} introduced a closed-form parallel projection that enforces both linear equality and inequality constraints while provably retaining the network's expressivity. \revised{Beyond linear constraints, ECO~\cite{goertzen2025eco} enforces a convex quadratic constraint in closed form.}

While closed-form projections offer efficient implementation and guaranteed constraint satisfaction, there are problems for which, inevitably, a single closed-form projection onto the feasible set does not exist, in which case iterative methods are often applied. 
\revised{Alternating projection methods are effective when the constraint set admits a decomposition into simpler sets with tractable projections}. PiNet \cite{grontas2026pinet} uses Douglas--Rachford to alternate projections onto convex sets whose Euclidean projections are known. \revised{LMI-Net \cite{tang2026lmi} satisfies linear matrix inequality (LMI) constraints by alternating between an affine equality projection and the projection onto the positive semidefinite cone, both of which can be computed in closed form.}
Other alternating projection methods use Dykstra's projection algorithm with success in solving optimization problems with linear constraints \cite{chen2018approximating,cristian2023end}. 
\revised{In cases where a convenient decomposition of the constraint set is unavailable, approaches rely on gradient-based updates to move network outputs toward constraint satisfaction.} 
DC3 \cite{donti2021dc3} takes gradient steps along the manifold of equality constraint satisfaction to satisfy inequality constraints, \revised{although convergence guarantees are limited to linear constraints}. For arbitrary nonlinear equality constraints, ENFORCE \cite{lastrucci2025enforce} iteratively projects onto local linearizations of the constraints with a closed-form Euclidean projection \cite{chen2024physics}, but does not handle inequality constraints. 

\revised{To address the limitations of current constraint enforcement methods, we introduce \name, a general equality and inequality constraint satisfaction approach. Unlike methods that rely on the existence of a tractable projection or decomposition of the feasible set, \name\ handles constraints of any form, making it particularly well-suited for general nonlinear constraint satisfaction.  \name\ iteratively enforces constraints using a damped projection based on local linearizations. In comparison to other gradient-based methods, our approach handles both inequality and equality constraints and provides theoretical convergence guarantees under standard regularity assumptions.}
Our specific contributions are:
\begin{enumerate}
    \item A differentiable projection framework that simultaneously enforces nonlinear equality and inequality constraints on neural network outputs.
    
    \item Convergence guarantees showing that the proposed iterative projection achieves arbitrarily small constraint violations.
    
    \item Experimental validation on a nonlinear MPC task demonstrating reliable constraint satisfaction with minimal degradation of optimal performance.
\end{enumerate}

\begin{figure}[t]
\centering

\includegraphics[width=1.0\linewidth]{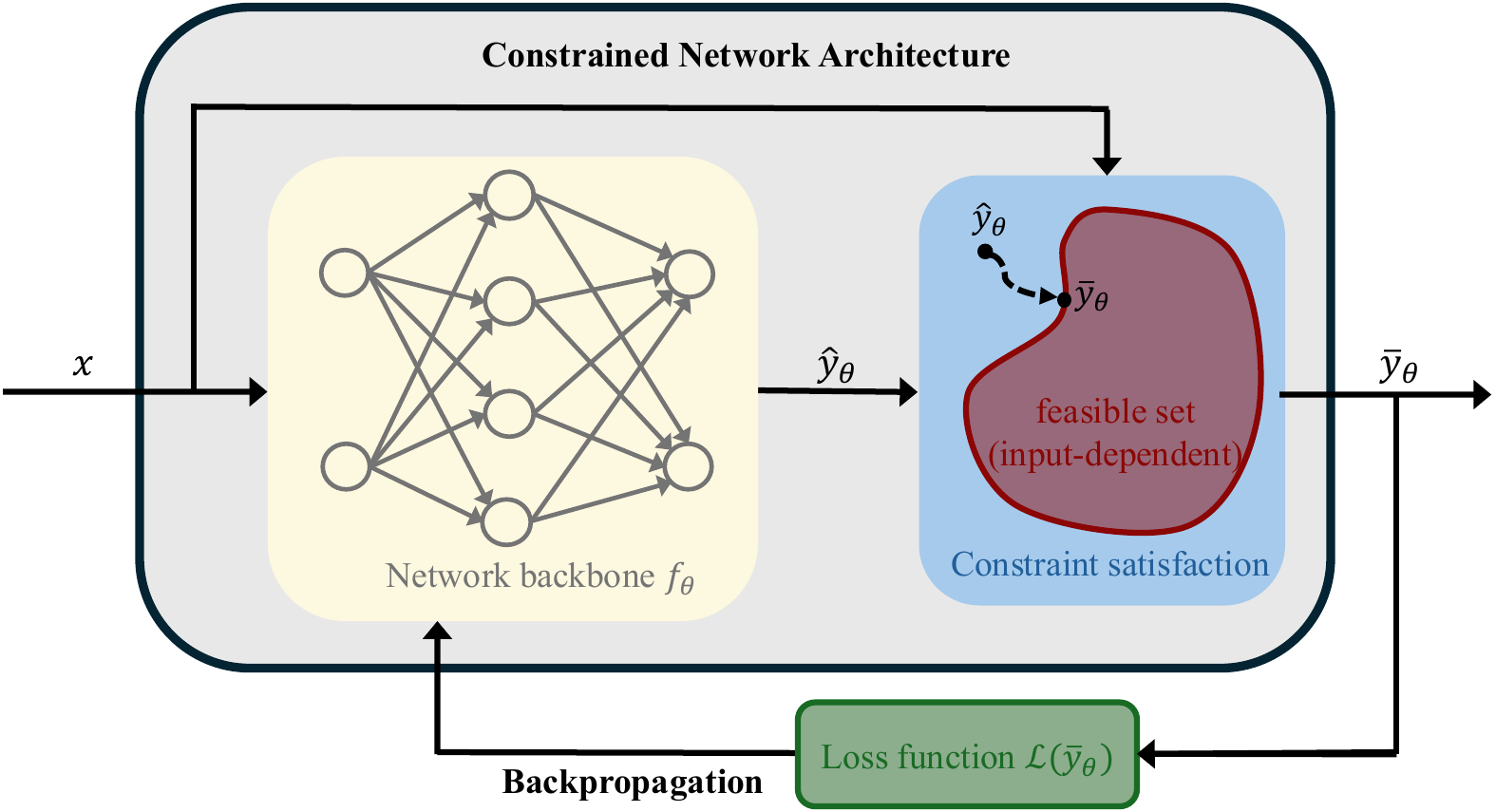}

\caption{Our model structure. The network output $\hat{y}_\theta$ is mapped to a point $\bar{y}_\theta$ on the feasible set through a differentiable procedure.}\label{fig:structure}

\end{figure}

\section{Problem Formulation}\label{sec:problem_form}
As discussed in Section~\ref{sec:intro}, the ability to incorporate critical domain knowledge into neural networks would improve their utility in scientific applications by ensuring physical fidelity or adherence to safety constraints. The goal of this work, therefore, is to develop a methodology that ensures neural network outputs satisfy nonlinear constraints \revised{whose projections cannot be computed in closed form}. Specifically, we map the unconstrained neural network output to a point that simultaneously satisfies nonlinear equality and inequality constraints, as in Figure~\ref{fig:structure}. Consider a neural network output $\hat{y}_\theta = f_\theta(x)$, where $x\in\mathbb{R}^p$ is the network input, $\hat{y}_\theta\in\mathbb{R}^n$ is the unconstrained output, and $\theta$ parameterizes the network. \revised{Let $c_x:\mathbb{R}^n\to\mathbb{R}^m$ denote $m$ input-dependent constraints of any form.} 
We aim to map $\hat{y}_\theta$ to $\bar{y}_\theta$ such that
\revised{\begin{equation}\label{eq:proj}
\bar{y}_\theta \in \left\{ y \in \mathbb{R}^n : b^l_x\le c_x(y) \le b^u_x\; \right\}.
\end{equation}}

\revised{Here, $b^u_x\in\mathbb{R}^m$ and $b^l_x\in\mathbb{R}^m$ are input-dependent upper and lower bounds. Note that the constraint form in \eqref{eq:proj} is general and includes equality constraints by setting their corresponding entries of $b_x^l$ and $b_x^u$ equal to each other}. 
Throughout this paper, we measure constraint violation with the residual
\revised{\[r(y):= \mathrm{ReLU}\left(c(y)-b^u\right)-\mathrm{ReLU}\left(b^l-c(y)\right),\]}
where the $\mathrm{ReLU}(\cdot) =\mathrm{max}(0,\cdot)$ operator is applied elementwise. We drop the dependence of \revised{$c$, $b^l$, and $b^u$} on the network input $x$ for notational convenience, although we emphasize that this method handles input-dependent constraints, as demonstrated in Section~\ref{sec:results}. \revised{Note that for an equality constraint $c_i$, where $i\in\{0,1,...,m\}$ denotes the index of a single constraint, $b^u_i=b^l_i$, and the respective entry of the residual becomes $c_i(y)-b_i^u$.} The Jacobian of the constraints with respect to $y$ is $J_c(y)$ where $J_{c,ij} = \frac{\partial c_i(y)}{\partial y_j}$.
The Jacobian of the constraint residuals $r(y)$ is $J_r(y)$ where $J_{r,ij} = \frac{\partial r_i(y)}{\partial y_j}$. A subscript $k$ denotes the $k$-th iterate of an iterative method, where $k \in \{0,1,2,\dots\}$. 

\section{Constrained Learning Methodology}
\subsection{Constraint Enforcement via Damped Local Linearizations}
While no closed-form solution exists for projections onto general nonlinear constraints \revised{$b^l\leq c(y)\leq b^u$}, there exist methods that can project onto linear equality or inequality constraint sets in closed form. \revised{We leverage this by linearizing the constraints around a known point $y_0$, with $c(y_0+\delta )\approx c(y_0)+J_c(y_0)\delta$. A step $\delta $ can then be computed such that $ b^l\leq c(y_0)+J_c(y_0)\delta \leq b^u$, although this does not ensure that $b^l\leq c(y_0+\delta)\leq b^u $. Instead, the step $\delta $ is intended to move closer to satisfaction of the true constraint, motivating the use of an iterative update step.} 
Specifically, we iteratively define the linearized constraint

\begin{equation}\label{eq:linear}
b^l\leq c(y_k)+J_c(y_k)[y_{k+1}-y_k] \leq b^u,
\end{equation} 


\revised{where $y_k$ denotes the fixed value of the current iterate, and each iterate takes a step $\delta_k  = y_{k+1}-y_k$. Importantly, \eqref{eq:linear} is \textit{linear} in $y_{k+1}$, admitting a closed-form HardNet \cite{min2024hardnet} projection onto a feasible point, as long as the number of independent constraints is less than or equal to the number of variables.} That is, when $J_{c,k}:=J_c(y_k)$ has full row rank, there exists a closed-form projection onto a point $y_{k+1}$ that satisfies \eqref{eq:linear}. Specifically, the HardNet projection is
\begin{equation}\label{eq:HardNet}
    y_{k+1} = y_k -J_{c,k}^\top (J_{c,k}J_{c,k}^\top )^{-1}r(y_k).
\end{equation}
The projection defined by \eqref{eq:HardNet} is well-suited for linear constraints due to the parallel projection geometry inherent to HardNet. That is, points are projected parallel to the boundaries of the satisfied constraints, allowing satisfaction of violated constraints without violating these already-satisfied constraints. In the case of nonlinear constraints, where the constraint boundary does not necessarily extend across the entire domain of the variables, a strictly parallel projection may fail to move toward violated constraints. See Figure~\ref{fig:example_vanilla} for an example of this failure mode.

\begin{figure}[t]
\centering

\subfloat[Vanilla HardNet iterates \label{fig:example_vanilla}]{
\includegraphics[width=0.48\linewidth]{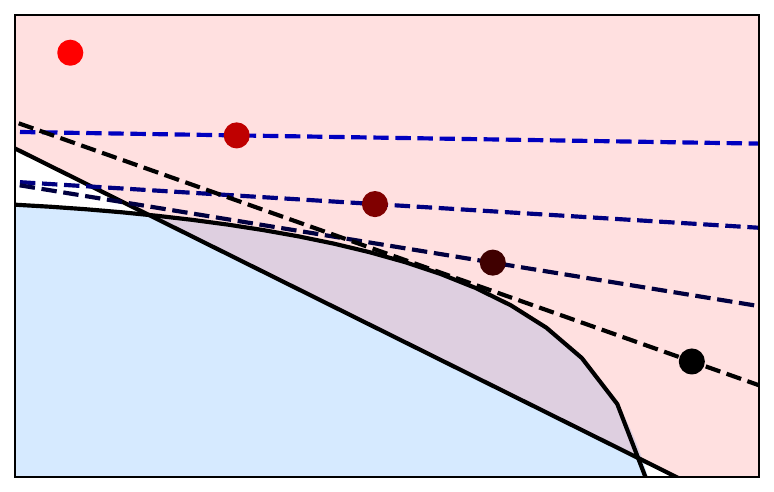}
}
\subfloat[\name\ iterates \label{fig:example_ours}]{
\includegraphics[width=0.48\linewidth]{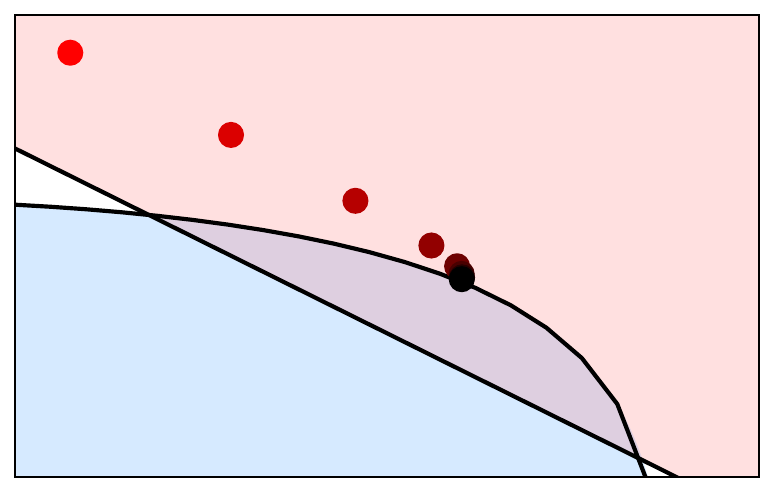}
}

\caption{Projection steps based on local linearizations under two projection methods. The feasible set is the overlap between the red and blue regions. The red point is the initial point, and points get darker in subsequent iterations. \revised{The dashed lines in (a) are the linearizations of the constraint defining the blue region.} An iterative projection using HardNet is forced to stay parallel to the boundary of the satisfied linear constraint, preventing it from reaching the feasible set. The damped method allows movement orthogonal to the constraint boundary, converging to the feasible set.}

\end{figure}

To circumvent this phenomenon inherent to a vanilla HardNet implementation, we add a damping term to our projection. 
\begin{equation}\label{eq:ours}
    \revised{y_{k+1}^\name} = y_k -J_{c,k}^\top (J_{c,k}J_{c,k}^\top +\varepsilon I)^{-1}\left[r(y_k)\right] \quad \varepsilon>0
\end{equation}
The regularization term, $\varepsilon I$, offers several key advantages. First, it damps the restriction on the direction orthogonal to satisfied constraints, allowing the projection method to converge to the feasible set. See Figure~\ref{fig:example_ours} for an example of this improvement. It also allows $(JJ^\top +\varepsilon I)x=b$ to have a unique solution for $x$, regardless of the number of constraints and variables. That is, $JJ^\top +\varepsilon I$ is invertible, even when $JJ^\top $ is not invertible, although in practice we use efficient linear solvers to compute the projection in \eqref{eq:ours} rather than explicitly computing the inverse. Finally, the addition of the $\varepsilon I$ term behaves as a regularizer, penalizing steps far away from the region in which the linear approximation is valid. This particular effect is discussed further in the next section. 

Importantly, this iterative procedure is differentiable. This enables an end-to-end training scheme, where the constrained output is considered during training, rather than only being applied at inference time. This is important because training seeks to ensure the constrained output minimizes the loss function, rather than the unconstrained output, which could be pushed away suboptimally when constraints are enforced only during inference. 











\subsection{\revised{Characterization via} Levenberg-Marquardt}
The damped projection method \revised{can be further characterized through its connection} to the Levenberg-Marquardt damped least-squares algorithm \cite{levenberg1944method}.
\revised{Both methods use damped updates based on iterative local linearizations, but they differ fundamentally in how they treat constraints that are already satisfied. In our setting, Levenberg-Marquardt could be applied to minimize the squared constraint residual $\|r(y)\|_2^2=\sum_i r_i(y)^2$.

While directly minimizing this residual}
does not necessarily admit a closed-form solution, minimizing a local linearization does. 
\revised{In particular, consider the linearized objective with a damping term penalizing large steps:
\begin{align}\label{eq:linearized_obj}
    \delta_{k} := y_{k+1}-y_k &=\underset{\delta}{\text{argmin }} \|r_k + J_{r,k}\delta\|_2^2 + \varepsilon\|\delta\|_2^2,
\end{align}
where $r_k:=r(y_k)$} for notational convenience. Importantly, the Jacobian $J_{r,k}$ here is the Jacobian of the \textit{constraint residual}, rather than the constraint itself. Taking the derivative of the objective in \eqref{eq:linearized_obj} and setting it to 0 gives \revised{the unique solution} 
\revised{
\begin{equation}\label{eq:LM}
    \delta_k = -(J_{r,k}^\top J_{r,k}+\varepsilon I)^{-1}J_{r,k}^\top r_k
    =-J_{r,k}^\top (J_{r,k}J_{r,k}^\top +\varepsilon I)^{-1}r_k,
\end{equation}
where the last equality follows from $(J_{r,k}^\top J_{r,k}+\varepsilon I)^{-1}J_{r,k}^\top = J_{r,k}^\top (J_{r,k}J_{r,k}^\top +\varepsilon I)^{-1}$ with the identity matrices understood to have the appropriate dimensions.


Thus, the Levenberg-Marquardt update in~\eqref{eq:LM} has the same algebraic form as our projection in~\eqref{eq:ours}, but with $J_{r,k}$ in place of $J_{c,k}$. This distinction is important: using the residual Jacobian causes already-satisfied inequality constraints to drop out of the update, whereas using the constraint Jacobian preserves their local geometry. As a result, the two methods induce \textit{fundamentally different projection directions}. To elaborate, we first characterize our projection as solving an optimization problem analogously:
\begin{theorem}
\label{thm:ours_opt}
At each iteration $k$, \name{} in~\eqref{eq:ours} solves
\begin{align} \label{eq:ours_opt}
    \delta_k^\name =\underset{\delta}{\text{argmin }} \|\tilde{r}_k(y_k) + J_{\tilde{r}_k}(y_k)\delta\|_2^2 +\varepsilon \|\delta\|_2^2,
\end{align}
where the modified residual is defined as
\begin{equation}
    \tilde{r}_k(y) := c(y)-\clip\big(c(y_k), b^l,b^u\big),
\end{equation}
with the elementwise clipping $\clip(v, v_\text{min}, v_\text{max}):=\max(v_\text{min},\min(v_\text{max}, v))$.

For all $y$ in a neighborhood of $y_k$, this modified residual differs from the original residual $r(y)$ in that constraints already satisfied at $y_k$ are treated as local equality constraints:
\begin{equation} \label{eq:new_residual}
    \tilde{r}_{k,i}(y) = \begin{cases}
        c_i(y)-c_i(y_k) & \text{if } b^l_i\leq c_i(y_k)\leq b^u_i\\
        r_i(y) & \text{otherwise}
    \end{cases},
\end{equation}
for each $i$-th component.
\end{theorem}
\begin{proof}
We provide a proof sketch. Similarly to the equivalence of \eqref{eq:linearized_obj} and \eqref{eq:LM} in the Levenberg-Marquardt algorithm, \name{} in~\eqref{eq:ours} is equivalent to
\begin{equation}
    \delta_k^\name = \underset{\delta}{\text{argmin }} \|r_k + J_{c,k}\delta\|_2^2 +\varepsilon \|\delta\|_2^2.
\end{equation}
This is equivalent to the more interpretable form \eqref{eq:ours_opt} since the modified residual $\tilde{r}_k$ has the same Jacobian as the constraint ($J_{\tilde{r}_k}=J_{c}$) while attaining the same value as the original residual $r$ at $y_k$ ($\tilde{r}_k(y_k)=r_k$). The property \eqref{eq:new_residual} can be verified for each case.
\end{proof}

This characterization in~\eqref{eq:ours_opt} shows that \name{} minimizes a locally linearized, damped approximation of the squared modified residual $\tilde{r}^{(k)}$, while treating already-satisfied constraints as local equality constraints: the objective $\|\tilde{r}^{(k)}(y)\|_2^2$ penalizes violations of $c_i(y)=c_i(y_k)$ through the term $\big(c_i(y)-c_i(y_k)\big)^2$ for each satisfied constraint $c_i$. In the Levenberg–Marquardt update, already-satisfied inequality constraints have both zero residual and zero residual Jacobian, so they are effectively ignored. In contrast, \name{} retains the Jacobian of every constraint. Thus, for satisfied constraints, the update still accounts for the local constraint geometry even when the current residual is zero. In the undamped HardNet projection, this forces the update direction parallel to the linearized boundaries of already-satisfied constraints. The damping term relaxes this restriction, allowing limited update in directions normal to those boundaries while still biasing the update toward preserving satisfaction. This softened geometry enables progress toward violated nonlinear constraints in settings where a strictly parallel projection may stall.
}

\subsection{Learning with Hard Constraints for Optimization}\label{sec:learned_opt}
The proposed methodology for enforcing constraints in neural network outputs is particularly relevant to learned optimization solvers. Consider an optimization problem of the form
\begin{equation}
    \underset{y}{\text{min }} Z_x(y)\quad{\text{subject to}\quad g_x(y) \le 0,\; h_x(y) = 0}
\end{equation}
that needs to be solved for different instances of $x$. The idea is to train a neural network to map an optimization context $x$ to a candidate solution $y_\theta$, rather than solving a constrained optimization independently at each instance. Often, the objective function $Z_x(y)$ itself is chosen as the loss function during training, \revised{rather than employing a supervised training approach, where the loss is defined by a distance between the predicted solution and a ground truth optimum}. Our methodology is well-suited for learned optimization problems because we can ensure the feasibility of the candidate solution $y_\theta$ via our iterative projection method.

\section{Convergence Guarantees}
In this section, we show that, under a standard set of assumptions, our projection method converges asymptotically to the feasible set. We define the violation energy as
\begin{equation}\label{eq:V}
V(y) := \frac{1}{2}\|r(y)\|_2^2.
\end{equation}
We establish two standard optimization conditions for our objective function
\(V(y)\).

\begin{assumption}
\label{assumption:smoothness}
{($L$-Smoothness)}: The violation energy $V(y)$ is $L$-smooth, meaning its gradients are $L$-Lipschitz continuous. 
\end{assumption}

Lemma~\ref{lem:quadratic_upper_bound} shows how this assumption provides a quadratic upper bound on the function's growth, controlling the region where our linear approximation is valid.

\begin{lemma} 
\label{lem:quadratic_upper_bound}

Let $V(y)$ be an $L$-smooth function on a convex domain, then the following quadratic upper bound holds for any $y_1, y_2$:
\[
V(y_2) \le V(y_1) + \nabla V(y_1)^\top  \Delta y + \frac{L}{2} \| \Delta y \|_2^2,
\]
where $\Delta y = y_2 - y_1$. (\cite{nesterov2013introductory}, Theorem 2.1.5)
\end{lemma}


It is worth noting that although Assumption~\ref{assumption:smoothness} is stated globally for notational simplicity, the proof only uses the quadratic upper bound from Lemma~\ref{lem:quadratic_upper_bound} along the iterates.
Therefore, it is sufficient that \(V\) admit an upper bound \(L\) on its local smoothness over any convex region containing the iterates (e.g., the initial sublevel set). 

\begin{assumption}
\label{assumption:PL}
{(P\L~Condition \cite{PL})} The function $V(y)$ with minimizer $y^\star$ satisfies the $\mu$-P\L~condition
\[
\forall y: \frac{1}{2} \|\nabla V(y)\|_2^2 \ge \mu \left( V(y) - V(y^\star) \right).
\]
\end{assumption}

\revised{This assumption connects gradients at any point to sub-optimality, allowing us to translate arguments from gradient space to sub-optimality. While it is a standard regularity assumption in nonconvex optimization, in our setting with~\eqref{eq:V}, a simple sufficient condition is a mild nondegeneracy condition on the violated constraints. Since already-satisfied inequality constraints do not contribute to $V(y)$ and $\nabla V(y)$, it suffices to consider the subset of violated and equality constraints. If the corresponding residual Jacobian has uniformly full row rank such that its transpose has minimum singular value uniformly bounded below by $\sigma_\text{min}>0$, then $\|\nabla V(y)\|_2^2=\|J_r(y)^\top r(y)\|_2^2\ge\sigma_\text{min}^2\|r(y)\|_2^2$, which implies the P\L~ condition with $\mu=\sigma_\text{min}^2$ whenever a feasbile point exists so that $V^*=0$.} 


\revised{
\begin{assumption}
\label{assumption:Lipschitz}
{(Lipschitzness of the Constraints)}
The constraint $c(y)$
is differentiable and \(G\)-Lipschitz on the domain of interest, i.e.,
\[
\|c(y_1)-c(y_2)\|_2 \le G\|y_1-y_2\|_2
\qquad \forall y_1,y_2.
\]
\end{assumption}
}

\revised{
 This assumption implies an upper bound on the Gram Matrix \(H_k = J_kJ_k^\top\), where \(J_k = J_c(y_k)\). Lemma~\ref{lem:bounded_eigenvals} gives a bound on the maximum eigenvalue of \(H_k\).
}

\revised{
\begin{lemma}
\label{lem:bounded_eigenvals}
Under Assumption~\ref{assumption:Lipschitz}, the maximum eigenvalue of \(H_k = J_kJ_k^\top\) satisfies
\[
\lambda_{\max}(H_k)\le G^2.
\]
\end{lemma}
}

\revised{
\begin{proof}
Fix \(y\) and let \(J(y)=J_c(y)\). For any unit vector \(v\), differentiability
and \(G\)-Lipschitzness of \(c\) imply
\[
\|J(y)v\|_2
=
\lim_{t\to 0}\frac{\|c(y+tv)-c(y)\|_2}{|t|}
\le G.
\]
Taking the supremum over all \(\|v\|_2=1\) gives \(\|J(y)\|_2\le G\). Therefore,
\[
\lambda_{\max}(H_k)
=
\|J_k\|_2^2
\le G^2.
\]
This proves the claim.
\end{proof}
}

\revised{
Under the aforementioned assumptions, Theorem~\ref{thm:convergence} shows that the violation energy decays exponentially fast as a function of the number of iterations.
}

\revised{
\begin{theorem}
\label{thm:convergence}
Consider the violation energy $V(y)=\frac{1}{2}\|r(y)\|_2^2$, with minimizer $y^\star$. Suppose \(V(y)\) is \(L\)-smooth and satisfies the \(\mu\)-P\L~condition, and the stacked constraint vector \(c(y)\) is \(G\)-Lipschitz. If the update step is defined as
\[
\Delta y_k = -J_k^\top (J_kJ_k^\top + \varepsilon I)^{-1} r(y_k),
\qquad J_k := J_c(y_k),
\]
then for any regularization parameter choice \(\varepsilon > \frac{L}{2}\), the violation energy decays as follows:
\[
V(y_{k+1})-V(y^\star)
\le
\bigl(1-2\mu c(\varepsilon)\bigr)\bigl(V(y_k)-V(y^\star)\bigr),
\]
where
\[
c(\varepsilon)
:=
\min\left\{
\frac{1}{\varepsilon}-\frac{L}{2\varepsilon^2},
\;
\frac{1}{\varepsilon+G^2}-\frac{L}{2(\varepsilon+G^2)^2}
\right\}
> 0.
\]
In particular, if there exists a feasible point, then 
\(V(y^\star)=0\) and
\[
V(y_k)\le \bigl(1-2\mu c(\varepsilon)\bigr)^k V(y_0),
\]
so \(V(y_k)\to 0\) exponentially fast.
\end{theorem}
}
\revised{
\begin{proof}
We begin by defining
$r_k := r(y_k)$, $J_k := J_c(y_k)$, and $H_k := J_kJ_k^\top$,
and substituting the update step
\[
\Delta y_k = -J_k^\top (H_k + \varepsilon I)^{-1}r_k
\]
into the quadratic upper bound:
\[
V(y_{k+1})
\le
V(y_k)+\nabla V(y_k)^\top \Delta y_k+\frac{L}{2}\|\Delta y_k\|_2^2.
\]
Now,
\[
\begin{split}
V(y)
&=
\frac{1}{2}\sum_i [\mathrm{ReLU}(c_i(y)-b^u)-\mathrm{ReLU}(b^l-c_i(y))]^2.
\end{split}
\]
Since the scalar function \(\phi(z)=\frac{1}{2}[\mathrm{ReLU}(z-b^u)-\mathrm{ReLU}(b^l-z)]^2\) is continuously differentiable with derivative \(\phi'(z)=\mathrm{ReLU}(z-b^u)-\mathrm{ReLU}(b^l-z)\), the chain rule gives
$
\nabla V(y_k)
=
J_k^\top r_k
$.
Therefore,
\[
\nabla V(y_k)^\top \Delta y_k
=
-r_k^\top H_k(H_k+\varepsilon I)^{-1}r_k,
\]
and
\[
\|\Delta y_k\|_2^2
=
r_k^\top (H_k+\varepsilon I)^{-1}H_k(H_k+\varepsilon I)^{-1}r_k.
\]
Substituting into the quadratic upper bound yields
\begin{align*}
V(y_{k}) - V(y_{k+1}) &\ge r_k^\top  H_k(H_k + \varepsilon I)^{-1}r_k \\ &- \frac{L}{2} r_k^\top  (H_k + \varepsilon I)^{-1} H_k (H_k + \varepsilon I)^{-1}r_k
\end{align*}
Let \(H_k = Q_k\Lambda_kQ_k^\top\) be the eigendecomposition of $H_k$, where
\[
\Lambda_k = \mathrm{diag}(\lambda_{k,1},\dots,\lambda_{k,m}),
\]
and let \(z_k = Q_k^\top r_k\). Lemma~\ref{lem:bounded_eigenvals} implies
$\forall i: 0\le \lambda_{k,i}\le G^2$. Therefore,
\[
V(y_k)-V(y_{k+1})
\ge
\sum_i
\left(
\frac{\lambda_{k,i}}{\lambda_{k,i}+\varepsilon}
-
\frac{L \cdot \lambda_{k,i}}{2(\lambda_{k,i}+\varepsilon)^2}
\right)
z_{k,i}^2.
\]
Let
$x_i = \frac{1}{\lambda_{k,i}+\varepsilon}$, and
$f(x)=x\left(1-\frac{L}{2}x\right)$.
Then,
$x_i \in \left[\frac{1}{\varepsilon+G^2},\,\frac{1}{\varepsilon}\right]$.
Since \(f\) is a concave quadratic, its minimum over this interval is attained at an endpoint. Therefore,
\[
f(x_i)\ge c(\varepsilon)
=
\min\left\{
\frac{1}{\varepsilon}-\frac{L}{2\varepsilon^2},
\;
\frac{1}{\varepsilon+G^2}-\frac{L}{2(\varepsilon+G^2)^2}
\right\}.
\]
Substituting this lower bound gives
\begin{align*}
V(y_k)-V(y_{k+1}) 
\ge c(\varepsilon)\sum_i \lambda_{k,i}z_{k,i}^2 
&= c(\varepsilon)\,r_k^\top H_k r_k \\
= c(\varepsilon)\,\|J_k^\top r_k\|_2^2 
&= c(\varepsilon)\,\|\nabla V(y_k)\|_2^2.
\end{align*}
Applying the \(\mu\)-P\L~condition yields
\[
V(y_k)-V(y_{k+1})
\ge
2\mu c(\varepsilon)\bigl(V(y_k)-V(y^\star)\bigr),
\]
or equivalently,
\[
V(y_{k+1})-V(y^*)
\le
\bigl(1-2\mu c(\varepsilon)\bigr)\bigl(V(y_k)-V(y^\star)\bigr).
\]
Iterating this inequality gives
\[
V(y_k)-V(y^*)
\le
\bigl(1-2\mu c(\varepsilon)\bigr)^k \bigl(V(y_0)-V(y^\star)\bigr).
\]
Importantly, the decay factor $1-2\mu c(\varepsilon)$ remains between $0$ and $1$. Because $\varepsilon > \frac{L}{2}$, we have $c(\varepsilon) > 0$ and therefore $1-2\mu c(\varepsilon)<1$. Additionally, since the global maximum of $f$ is $\frac{1}{2L}$ and $L \geq \mu$ \cite{karimi2016linear}, $f(x_i) \leq \frac{1}{2\mu}$. This guarantees that the decay factor $1 - 2\mu c(\varepsilon)$ remains non-negative, validating the monotonic decay. Finally, since \(V(y_{k+1})\le V(y_k)\), the iterates remain in the initial sublevel set of \(V\), so the proof works under a local smoothness bound over that region. This completes the proof.
\end{proof}
}

\section{Empirical Validation}\label{sec:results}
We demonstrate the ability of our proposed methodology to approximate a constrained optimization solver, as described in Section~\ref{sec:learned_opt}. 
We use an illustrative model predictive control example \revised{of a unicycle with nonlinear dynamics and an obstacle avoidance constraint on the state. 
\begin{equation}\label{eq:mpc}
    \begin{split}
        \underset{z,u}{\text{min }} & \sum_{k=0}^{N-1} ||z_{k+1}-z_{\text{target}}||_{W_k}^2 + ||u_k||_R^2\\ &z_k=[x_k,y_k,\theta_k]^\top, u=[v_k,\omega_k]^\top\\
        &\text{subject to}\\
        &x_{k+1}=x_k+\Delta t v_k \cos \theta_k\quad \forall k\in\{0,1,...N-1\}\\
        & y_{k+1}=y_k+\Delta t v_k \sin \theta_k\quad \forall k\in\{0,1,...N-1\}\\
        & \theta_{k+1}=\theta_k+\Delta t \omega_k \quad \forall k\in\{0,1,...N-1\}\\
        &b^l\leq u_k\leq b^u\quad \forall k\in\{0,1,...N-1\}\\
        & [x_k,\,\,y_k] Q[x_k,\,\,y_k]^\top \geq z_b \quad \forall k\in\{1,...,N\}\\
        & z_0 = z_{\text{in}}
    \end{split}
\raisetag{20pt}
\end{equation}
}
We use $W_k =\operatorname{diag}(1,1,1)$ for all nonterminal stages and $W_k=\operatorname{diag}(10,10,1)$ for the terminal stage ($k=N-1$) and $R = \operatorname{diag}(0.1,0.1)$. 
Box constraints are enforced on the controller $u$ with \revised{$b^l=[0,-1.5]^\top$ and $b^u=[2,1.5]^\top$, and the lower bound on the quadratic constraint is $z_b=1$. We use $Q=\text{diag}([0.51,0.91])$ to define the obstacle and $z_\text{target}=[3.5,0,0]^\top$.} We aim to find a mapping from the initial condition $z_{\text{in}}$ to the optimal state and outputs $z$ and $u$ across the horizon $N$. That is, we learn $f_\theta:\mathbb{R}^{n_z}\to\mathbb{R}^{N(n_u+n_z)}$ such that the dynamics constraints as well as the constraints on the states and controllers are satisfied. \revised{We use $N=10$, so the problem has 50 variables and 80 constraints (50 inequalities and 30 equalities).} We use a feed-forward neural network with 2 hidden layers and 200 neurons each. For the constraint satisfaction layer, we use $\varepsilon=0.3$ and $500$ linearized iterations. \revised{Rather than using a supervised learning approach,} we elect to directly minimize the MPC objective during training. 

We test our trained model on 100 random feasible initial conditions and compare to a soft-constrained baseline \revised{and DC3~\cite{donti2021dc3}. For DC3, we try the best hyperparameters (step size, regularization coefficient, and number of iterations) reported in~\cite{donti2021dc3} and perform hyperparameter tuning ourselves, and we report the most favorable results}. \revised{To evaluate solution quality, reference trajectories are computed by directly solving \eqref{eq:mpc} in CasADi~\cite{Andersson2019CasADi} using IPOPT~\cite{wachter2006implementation}, a standard large-scale nonlinear programming solver widely used in nonlinear MPC.} 
We calculate both an average and a maximum suboptimality across test samples, defined as $S = \frac{\max(0,F(\bar{z},\bar{u})-F(z^*,u^*))}{F(z^*,u^*)},$ where $F$ is the objective function in \eqref{eq:mpc}, $z^*$ and $u^*$ are the optimal state and control variables determined by the solver, and $\bar{z}$ and $\bar{u}$ are those output by the model. 
We report the average and maximum constraint residual across all constraints and all test cases in Table~\ref{tab:mpc_results}.


\begin{table}[h]
\caption{Average (avg) and maximum (max) suboptimality and constraint violations across 100 initial conditions}
\label{tab:mpc_results}
\setlength{\tabcolsep}{2pt}
\centering
\begin{tabular}{l|cccccc}
\toprule
\multicolumn{1}{c}{ }&
\multicolumn{2}{c}{\name} &
\multicolumn{2}{c}{DC3} &
\multicolumn{2}{c}{Soft-Constrained}\\
\cmidrule(r){2-3}
\cmidrule(l){4-5}
\cmidrule(l){6-7}
\multicolumn{1}{c}{} &
Avg & Max &
Avg & Max &
Avg & Max \\
\midrule
Suboptimality & \textbf{0.0137} & \textbf{0.181} & 0.136 & 1.22 & 0.248 & 1.23 \\
Dynamics & 1.34e-3 & 0.0578 & \textbf{0.0} & \textbf{0.0} &8.51e-3 & 0.209 \\
$u$ constraint  & 9.06e-5 & \textbf{9.89e-3} & \textbf{2.63e-5} & 0.0211 & 2.74e-4 & 0.109 \\
$z$ constraint  & \textbf{4.47e-5} & \textbf{5.74e-3} & 2.13e-4 & 0.0663 & 1.85e-4 &0.0446\\
\bottomrule
\end{tabular}
\end{table}

\revised{Figure~\ref{fig:x_results} shows the state trajectory of a closed-loop system, where the optimization problem in \eqref{eq:mpc} is solved (IPOPT MPC) or predicted (\name, DC3, and soft controller) at each state to select the control action. The \name\ controller drives the state toward the target, closely matching the true optimal trajectory. Importantly, the learned controller avoids the obstacle, demonstrating the ability of \name\ to enforce feasibility on network outputs. The DC3 controller also avoids the obstacle, but follows a non-optimal path toward the target. These results are consistent with Table~\ref{tab:mpc_results}. The \name\ controller achieves the lowest suboptimality while maintaining near-zero constraint violations across all constraints. In contrast, DC3 exhibits significantly higher suboptimality with larger worst-case inequality violations, and the soft-constrained baseline performs worst in terms of optimality and worst-case constraint violations. Overall, \name\ combines near-optimal performance with consistently tight constraint enforcement.}




\begin{figure}[t]
\centering

\includegraphics[width=0.9\linewidth]{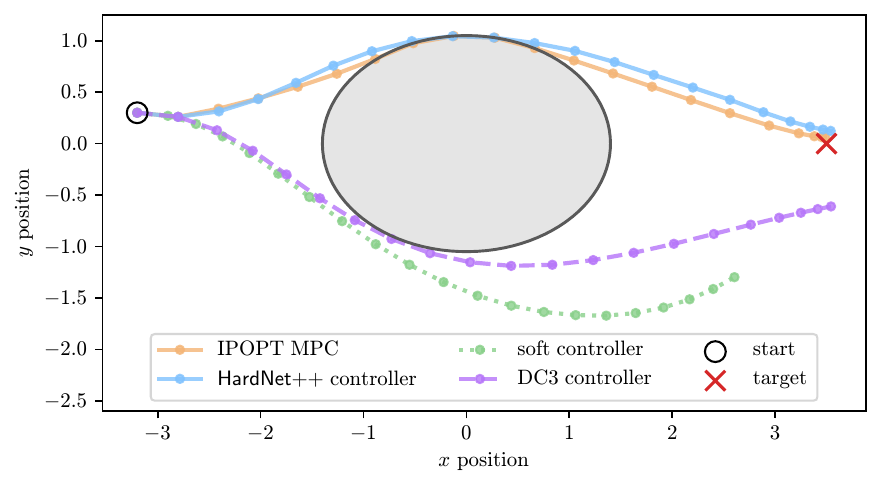}

\caption{\revised{Closed-loop state trajectories for learned controllers and a reference optimal controller for a sample initial condition. The feasible region is the region outside the grey ellipse. The \name\ controller aligns closely with the optimal controller, demonstrating the capacity of the constrained network to approximate the optimization solver.}}
\label{fig:x_results}

\end{figure}

\section{Conclusion}
To address the challenge of enforcing constraints in neural networks, we introduced \revised{\name}, a differentiable enforcement method for nonlinear inequality and equality constraint satisfaction. Unlike many existing approaches, this method can handle nonlinear inequality and equality constraints simultaneously and, under standard regularity assumptions, is theoretically guaranteed to converge to the feasible set. The theoretical and empirical results of this paper suggest that \revised{\name\ }provides a practical mechanism for incorporating hard constraints into network models, with applications in learned optimization and scientific machine learning. Future work could explore \revised{low-rank updates to the Jacobian matrix at each iteration to improve the efficiency of the procedure}, as well as applications of \revised {\name\ }to larger-scale control and learning problems.

\addtolength{\textheight}{-12cm}   




\bibliographystyle{IEEEtran}
\bibliography{refs}

\end{document}